\documentclass[a4paper,10pt]{cypart}
\usepackage{amsmath}
\usepackage{amsthm}
\usepackage{balance}
\usepackage{relsize}
\usepackage[linesnumbered,ruled,vlined]{algorithm2e}
\usepackage[nomath]{stix}

\newgeometry{left=22mm,right=22mm,top=26mm,bottom=34mm,columnsep=6mm} %

\def\boldd{\boldsymbol d}
\def\boldf{\boldsymbol f}
\def\boldg{\boldsymbol g}

\def\boldw{\boldsymbol w}
\def\boldx{\boldsymbol x}
\def\boldy{\boldsymbol y}
\def\boldz{\boldsymbol z}

\newcommand\dac[4]{{\mathlarger\alpha}_{#1,#2\to #3,#4}}
\newcommand\ec[1]{{\mathlarger\delta}_{#1}}
\newcommand\eqdef{\overset{\mbox{def}}{=\joinrel=}}
\newcommand\csup[2]{#1^{[#2]}}

\newtheorem{thm}{Theorem}

\newtheorem{cor}[thm]{Corollary}

\title{\LARGE\bf Derivation of the Backpropagation Algorithm Based on Derivative Amplification Coefficients}
\author{Yiping Cheng}
\affil{\vskip -3pt School of Electronic and Information Engineering\\
Beijing Jiaotong University, Beijing 100044, China\\
\texttt{ypcheng@bjtu.edu.cn}}

\begin{abstract}
The backpropagation algorithm for neural networks is widely felt hard to understand, despite the existence of some well-written explanations and/or derivations. This paper provides a new derivation of this algorithm based on the concept of derivative amplification coefficients. First proposed by this author for fully connected cascade networks, this concept is found to well carry over to conventional feedforward neural networks and it paves the way for the use of mathematical induction in establishing a key result that enables backpropagation for derivative amplification coefficients. Then we establish the connection between derivative amplification coefficients and error coefficients (commonly referred to as errors in the literature), and show that the same backpropagation procedure can be used for error coefficients. The entire derivation is thus rigorous, simple, and elegant.
\end{abstract}


\begin{document}

\maketitle

\section{Introduction}
The backpropagation algorithm is extremely efficient for computing the gradient of error functions in machine learning using neural networks. Combined with the stochastic gradient descent method for optimization, it has long been the cornerstone for efficient training of neural networks. However, this algorithm is widely considered by practitioners in the field to be complicated, as they often feel frustrated when they try to have a full understanding of this algorithm. Sure, to be fair, there are several well-written explanations and/or derivations of backpropagation, such as the ones in \cite{haykin99,nielsen15}, but there are still a large portion of readers who still think they are confused, which include me. To be frank, I was nearly totally lost when reading the derivation in \cite{haykin99}, and I also felt frustrated at reading the ``How backpropagation works'' chapter in \cite{nielsen15} since it requires an advanced mathematical tool ``Hadamard product''. I wish to have a simple derivation and I believe there must be one.

I am not a pure freshman to neural networks. In fact, I have some research experience in an unconventional neural network architecture -- fully connected cascade networks, and I have published a paper \cite{cheng2017a} in that area. In that paper I proposed a backpropagation algorithm for the gradient of fully connected cascade networks. However, when I wrote that paper I did not realize that the concepts proposed in it can be used to write a new derivation of the well-known backpropagation algorithm for conventional feedforward neural networks (multilayer perceptrons). It was only recently that I came to the idea that the concept of {\em derivative amplification coefficients\/} may carry over to multilayer perceptrons and a new derivation can possibly be written which may hopefully be more accessible than the existing derivations to a wide audience among neural network practitioners. So this paper is the result of that effort of mine.

The rest of this paper is structured as follows. Section 2 introduces neural networks as nonlinear functions, which also sets out the notation. Section 3 explains why we need to compute the partial derivatives. Section 4 defines derivative amplification coefficients for multilayer perceptrons. Section 5 establishes key results that enable backpropagation for derivative amplification coefficients. Section 6 defines error coefficients and shows that the same backpropagation can be used for error coefficients, which actually appear in the algorithm. We conclude this paper with Section 7.

\section{Neural Network as a Nonlinear Function}

In this paper neural networks are understood as conventional feedforward neural networks, i.e. multilayer perceptrons. There is nothing mysterious in a neural network. It is merely a nonlinear function composed of layers, as depicted below.

\begin{center}
\includegraphics[width=0.4\textwidth]{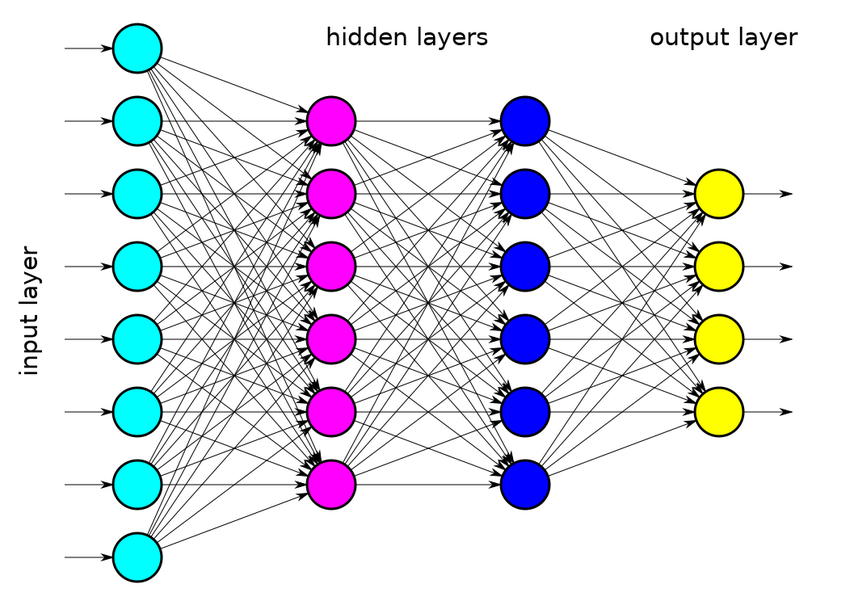}
\end{center}

There is an input layer, a number of hidden layers, and an output layer in a neural network. Each hidden layer receives input from its previous layer and sends its output to its next layer. The input layer has no previous layer and the output layer has no next layer. The input layer does not do any processing, thus it is not counted into the number of layers of a neural network. Therefore when we speak of a ``three-layer neural network'', we mean that the number of hidden layers in the neural network is two.

Consider a neural network with $L$ layers. We designate the input layer as the $0$-th layer, the hidden layers as the $1$-st to $(L-1)$-th layers, and the output layer as the $L$-th layer. Let us denote the input, which is a vector, by \begin{equation} \boldx =(x_1,\cdots,x_n). \end{equation}

We denote the output of the $l$-th layer, which is again a vector, by \begin{equation}\csup{\boldz}{l}=(\csup{z}{l}_1,\cdots,\csup{z}{l}_{h_l}).\end{equation} Then we have \begin{equation}
\csup{\boldz}{0} = \boldx,\quad h_0=n.
\end{equation} Let $q$ be the output dimension, then \begin{equation}
h_L=q.
\end{equation}

Each hidden layer and the output layer produces its output in two steps. The first step is weighted sum and the second step is activation. Let us denote the $l$-th layer weighted sum by \begin{equation}\csup{\boldy}{l}=(\csup{y}{l}_1,\cdots,\csup{y}{l}_{h_l}).\end{equation} Then for each $l$ with $1\le l\le L$, $i$ with $1\le i\le h_l$, we have \begin{equation}
\csup{y}{l}_i=w_{l,i,0}+\sum_{j=1}^{h_{l-1}} w_{l,i,j}\thinspace \csup{z}{l-1}_j
\end{equation}\begin{equation}
\csup{z}{l}_i=\phi_l(\csup{y}{l}_i)
\end{equation} where $\phi_l$ is the activation function for the $l$-th layer. All hidden layers typically use the same activation function, which must be nonlinear. The output layer will typically use a different activation function from that of the hidden layers and is dependent upon the purpose of the network.

Eqs. (6,7) describe a processing unit which processes input from the previous layer to a component of the output vector of this layer. It is called a node, or a neuron, and is represented in the above figure by a dot. Thus the $l$-th layer has $h_l$ neurons. The power of neural networks lies in the fact that a neural network can approximate any smooth function arbitrarily well, provided that it has at least one hidden layer and has a sufficient number of neurons.

The weights $w_{l,i,j}$ are the real parameters of the neural network, and they, together with the integer parameters $L$ and $n, h_1,\cdots,h_L$, uniquely determine the behavior of the network, i.e. the function that it represents. For notational convenience, let us now group the weights into a vector
\begin{equation}\boldw = (\boldw_1,\cdots,
\boldw_L)\end{equation} where for each $l=1,\ldots,L$,
\begin{equation}
\boldw_l =
(\boldw_{l,1},\cdots,\boldw_{l,h_l})
\end{equation} where for each $i=1,\ldots,h_l$,\begin{equation}
\boldw_{l,i} =
(w_{l,i,0},w_{l,i,1},\cdots,w_{l,i,h_{l-1}}).
\end{equation}

Thus, $\boldw$ consists of $\sum\limits_{l=1}^L (1+h_{l-1})h_l$ weights.

\section{The Need for the Partial Derivatives}

Since a neural network defines a multi-input multi-output function $\boldf(\boldx;\boldw)$ ($\boldx$ being the input vector and $\boldw$ being the vector of weights), it has a Jacobian at any point in the Euclidean space, with respect to both $\boldx$ and $\boldw$. Computing of the partial derivatives that form the Jacobian with respect to $\boldw$ is important for all applications of neural networks, as we will shortly show. Let us here take the regression application of neural networks as example.

A regression problem is given a set of input-output data $\{(\boldx^{(k)},\boldd^{(k)})\}_{k=1}^N$, where $N$ is the number of data points, $\boldx^{(k)}\in\mathbb{R}^n$ and $\boldd^{(k)}\in\mathbb{R}^q$, to find a function $\boldf$ from a particular function space that best matches the data. Now suppose that we have chosen the function space to be neural networks with $L$ layers and $h_1,\cdots,h_{L-1}$ neurons for the hidden layers respectively. Now it remains to find the weights. Reasonably, an optimal set of weights should minimize the following error function: \begin{equation}
E(\boldw)={1\over 2}\sum_{k=1}^N ||\boldf(\boldx^{(k)};\boldw)-\boldd^{(k)}||^2={1\over 2}\sum_{k=1}^N \sum_{o=1}^{q} [f_o(\boldx^{(k)};\boldw)-d^{(k)}_o]^2.
\end{equation}

Using such a criterion, the regression problem reduces to an optimization problem. There are two classes of optimization methods: gradient-based and non-gradient-based. For neural network regression problems, gradient-based methods are much faster than non-gradient-based methods. In fact, the method currently prevalent in neural network optimization is {\em stochastic gradient descent\/}, abbreviated as SGD. In SGD, an initial weight vector is randomly chosen and the algorithm continuously updates it incrementally in the direction of the steepest gradient descent based on the newly arrived data $(\boldx^{(k)},\boldd^{(k)})$. The gradient of (11) is additive, i.e., the gradient of the error involving many data samples is the sum of the gradients of the errors each involving a single data sample. Therefore in the sequel let us consider only one data sample $(\boldx^{(k)},\boldd^{(k)})$. In addition, for sake of notational brevity, let us also drop the $\cdot^{(k)}$ superscript. So let us assume that we have data $(\boldx,\boldd)$, the current weight vector is $\boldw$, and the error function is \begin{equation}
E(\boldx;\boldw)={1\over 2}||\boldf(\boldx;\boldw)-\boldd||^2={1\over 2}\sum_{o=1}^{q} [f_o(\boldx;\boldw)-d_o]^2.
\end{equation}

Then the gradient of $E$ with respect to $\boldw$ is composed of the following partial derivatives: \begin{equation}
{\partial E\over \partial w_{l,i,j}}(\boldx;\boldw)=\sum_{o=1}^{q} [f_o(\boldx;\boldw)-d_o]{\partial f_o\over\partial w_{l,i,j}}(\boldx;\boldw).
\end{equation}

We then see that computing the partial derivatives of the error with respect to the weights boils down to computing the partial derivatives of the network output functions with respect to the weights.

\section{Definition of Derivative Amplification Coefficients}

It is perceived by the present author that the existing derivations often confuse functions with their output variables. It is believed by the present author that the variable is not an official mathematical concept but the function is, and the chain rule should be applied to functions, not variables. It is the opinion of the present author that purported applications of the chain rule to variables might make the results problematic or unconvincing, and this is especially the case as neural network functions involve many levels of composition. Therefore, to be mathematically precise and rigorous, we choose to adopt a pure function point of view. So in the following, we define the functions that map $(\boldx;\boldw)$ to $\csup{y}{l}_i$ and $\csup{z}{l}_i$. It is done in a recursive manner.

For each $i$ with $1\le i\le n$,
\begin{equation}
g^{[0]}_i(\boldx;\boldw)\eqdef x_i.
\end{equation} And for each $l$ with $1\le l\le L$, $i$ with $1\le i\le h_l$, \begin{equation}
f^{[l]}_i(\boldx;\boldw)\eqdef w_{l,i,0}+\sum_{j=1}^{h_{l-1}} w_{l,i,j}\thinspace g^{[l-1]}_{j}(\boldx;\boldw),
\end{equation}\begin{equation}
g^{[l]}_i(\boldx;\boldw)\eqdef \phi_l(f^{[l]}_i(\boldx;\boldw)).
\end{equation}

The output variables of the above functions with $(\boldx;\boldw)$ as input are so frequently used in this paper that we have to give them short notations: \begin{equation}
\csup{y}{l}_i\eqdef f^{[l]}_i(\boldx;\boldw)\quad\mbox{and}\quad\csup{\boldy}{l}=(\csup{y}{l}_1,\cdots,\csup{y}{l}_{h_l}),
\end{equation} and \begin{equation}
\csup{z}{l}_i\eqdef g^{[l]}_i(\boldx;\boldw)\quad\mbox{and}\quad\csup{\boldz}{l}=(\csup{z}{l}_1,\cdots,\csup{z}{l}_{h_l}).
\end{equation}

These two notations (17--18) are also consistent with our informal notations (2--7).

Now we are in a position to define the derivative amplification coefficients. Like $\csup{y}{l}_i$ and $\csup{z}{l}_i$, they are also outputs of functions with $(\boldx;\boldw)$ as input, so they vary as $(\boldx;\boldw)$ varies.

For all $1\le l\le r\le L$ and $1\le i\le h_l$, $1\le t\le h_r$, the {\em derivative amplification coefficient\/} from node $(l,i)$ to node $(r,t)$, denoted by $\dac{l}{i}{r}{t}$, is defined recursively as follows:\begin{equation}
\dac{l}{i}{l}{t}\eqdef\left\{\begin{array}{ll}
1, & \mbox{if }i=t, \\
0, & \mbox{otherwise;}\end{array}\right.
\end{equation} and if $\thinspace l<r$,
\begin{equation}
\dac{l}{i}{r}{t}\eqdef\sum_{j=1}^{h_{r-1}} w_{r,t,j}\thinspace\phi'_{r-1} (\csup{y}{r-1}_j)\thinspace \dac{l}{i}{r-1}{j}.
\end{equation}

The role that derivative amplification coefficients play in the computation of partial derivatives is seen in the following theorem.

\begin{thm}
For $1\le l\le r\le L$ and $1\le i\le h_l$, $0\le j\le h_{l-1}$, $1\le t\le h_r$, \begin{equation}
{\partial f^{[r]}_t\over \partial w_{l,i,j}}(\boldx;\boldw)=\dac{l}{i}{r}{t}\cdot {\partial f^{[l]}_i\over \partial w_{l,i,j}}(\boldx;\boldw).
\end{equation}
\end{thm}

\begin{proof} The result is proved by mathematical induction on $r$. For the base case, i.e. $r=l$, the result obviously holds. Now as induction hypothesis suppose that the result holds when $r=p\ge l$. Then for all $1\le t\le h_{p+1}$, by (15,16), \[f^{[p+1]}_t(\boldx;\boldw)= w_{p+1,t,0}+\sum_{s=1}^{h_{p}} w_{p+1,t,s}\thinspace \phi_p(\thinspace\underbrace{f^{[p]}_{s}(\boldx;\boldw)}_{\csup{y}{p}_s}\thinspace).\]

Note that in the above equation the index of summation is changed from $j$ to $s$, in order to avoid symbol collision with the $j$ in $w_{l,i,j}$. Then \begin{align*}{\partial f^{[p+1]}_t\over \partial w_{l,i,j}}(\boldx;\boldw) &= \sum_{s=1}^{h_{p}} w_{p+1,t,s}\thinspace \phi'_{p}(\csup{y}{p}_s)\thinspace {\partial f^{[p]}_{s}\over \partial w_{l,i,j} }(\boldx;\boldw) & \mbox{by chain rule}\\
&= \sum_{s=1}^{h_{p}} w_{p+1,t,s}\thinspace\phi'_{p}(\csup{y}{p}_s)\thinspace \dac{l}{i}{p}{s}\cdot {\partial f^{[l]}_i\over \partial w_{l,i,j}}(\boldx;\boldw) & \mbox{by induction hypothesis}\\
&= \dac{l}{i}{p+1}{t}\cdot {\partial f^{[l]}_i\over \partial w_{l,i,j}}(\boldx;\boldw). & \mbox{by definition (20)}
\end{align*}

So the result also holds for $r=p+1$, and the inductive step is complete.
\end{proof}

As $\csup{\boldz}{L}$, rather than $\csup{\boldy}{L}$, is generally considered the genuine network output, most readers will feel strange, why our derivative amplification coefficients are about $f$, i.e, about $y$, not $z$? The answer is that while we may also define derivative amplification coefficients so that they are about $z$, the mathematical formulation will be less elegant, and the resulting algorithm will also be different from the commonly known backpropagation algorithm. In fact, we believe that it is more helpful to directly view $\csup{\boldy}{L}$ as the output of the network.

\section{Backpropagation for Derivative Amplification Coefficients}

Apparently, $\dac{l}{i}{L}{o}$ can be used for computing ${\partial f^{[L]}_o\over\partial w_{l,i,j}}(\boldx;\boldw)$. There are $q\sum\limits_{l=1}^{L-1} h_l$ such derivative amplification coefficients. What is more, if we compute them naively based on the definition, then there is a huge computational burden. Thus, it is extremely desirable to find any structural relationships between the coefficients in the hope that the computational cost can be reduced. For inspiration let us look at a 2-3-2-1 network. For this network we have

\[f^{[2]}_1(\boldx;\boldw)=w_{2,1,0}+\sum_{j=1}^{3} w_{2,1,j}\thinspace \phi_1(\thinspace\underbrace{f^{[1]}_j(\boldx;\boldw)}_{\csup{y}{1}_j}\thinspace),\]
\[f^{[2]}_2(\boldx;\boldw)=w_{2,2,0}+\sum_{j=1}^{3} w_{2,2,j}\thinspace \phi_1(\thinspace\underbrace{f^{[1]}_j(\boldx;\boldw)}_{\csup{y}{1}_j}\thinspace),\]
\[f^{[3]}_1(\boldx;\boldw)=w_{3,1,0}+\sum_{j=1}^{2} w_{3,1,j}\thinspace \phi_2(\thinspace\underbrace{f^{[2]}_j(\boldx;\boldw)}_{\csup{y}{2}_j}\thinspace).\]

And \[\dac{1}{2}{2}{1}=w_{2,1,2} \thinspace\phi'_1(\csup{y}{1}_2),\]
\[\dac{1}{2}{2}{2}= w_{2,2,2}\thinspace \phi'_1(\csup{y}{1}_2),\]
\[\dac{1}{2}{3}{1}=\sum_{j=1}^2 w_{3,1,j}\thinspace\phi'_2 (\csup{y}{2}_j)\thinspace \dac{1}{2}{2}{j} = w_{3,1,1}\thinspace\phi'_2 (\csup{y}{2}_1)\thinspace w_{2,1,2} \phi'_1(\csup{y}{1}_2) + w_{3,1,2}\thinspace\phi'_2 (\csup{y}{2}_2)\thinspace w_{2,2,2} \phi'_1(\csup{y}{1}_2),\]
\[\dac{2}{1}{3}{1}=w_{3,1,1}\thinspace \phi'_2(\csup{y}{2}_1),\]
\[\dac{2}{2}{3}{1}=w_{3,1,2}\thinspace \phi'_2(\csup{y}{2}_2).\]

It is a simple matter to verify that \[\dac{1}{2}{3}{1}=\phi'_1(\csup{y}{1}_2)\thinspace ( w_{2,1,2}\thinspace\dac{2}{1}{3}{1}+ w_{2,2,2}\thinspace\dac{2}{2}{3}{1}).\]

It turns out that this equation is not a mere coincidence, and in general we have the following theorem.

\begin{thm}
For $\thinspace 1\le l<r\le L$ and $1\le i\le h_{l}$, $1\le t\le h_r$, \begin{equation}
\dac{l}{i}{r}{t}=\phi_{l}'(\csup{y}{l}_i)\sum_{s=1}^{h_{l+1}} w_{l+1,s,i}\thinspace \dac{l+1}{s}{r}{t}.\end{equation}
\end{thm}

\begin{proof} The result is proved by mathematical induction on $r$.

Base case, i.e. $r=l+1$: By definition (19,20),

\[\dac{l}{i}{r}{t}=\dac{l}{i}{l+1}{t}=\sum_{j=1}^{h_{l}} w_{l+1,t,j}\thinspace\phi'_{l} (\csup{y}{l}_i)\thinspace \dac{l}{i}{l}{j}=w_{l+1,t,i}\thinspace\phi'_{l} (\csup{y}{l}_i),\]

\[\phi_{l}'(\csup{y}{l}_i)\sum_{s=1}^{h_{l+1}} w_{l+1,s,i}\thinspace \dac{l+1}{s}{r}{t}=\phi_{l}'(\csup{y}{l}_i)\sum_{s=1}^{h_{l+1}} w_{l+1,s,i}\thinspace \dac{l+1}{s}{l+1}{t}=w_{l+1,t,i}\thinspace\phi'_{l} (\csup{y}{l}_i).\]

So the result holds for the base case.

Inductive step: Suppose that the result holds when $r=p\ge l+1$. That is, for $1\le i\le h_{l}$, $1\le j\le h_{p}$, we have \begin{equation}\dac{l}{i}{p}{j}=\phi_{l}'(\csup{y}{l}_i)\sum_{s=1}^{h_{l+1}}  w_{l+1,s,i}\thinspace \dac{l+1}{s}{p}{j}.\end{equation}

Then \begin{align*}
\dac{l}{i}{p+1}{t}&=\sum_{j=1}^{h_{p}} w_{p+1,t,j}\thinspace \phi'_{p} (\csup{y}{p}_j)\thinspace\dac{l}{i}{p}{j} & \mbox{by definition (20)}\\
&=\sum_{j=1}^{h_{p}} w_{p+1,t,j}\thinspace \phi'_{p} (\csup{y}{p}_j)\thinspace\phi_{l}'(\csup{y}{l}_i)\sum_{s=1}^{h_{l+1}}  w_{l+1,s,i}\thinspace \dac{l+1}{s}{p}{j} & \mbox{by induction hypothesis (23)}\\
&=\phi_{l}'(\csup{y}{l}_i)\sum_{j=1}^{h_{p}}\sum_{s=1}^{h_{l+1}} w_{p+1,t,j}\thinspace \phi'_{p} (\csup{y}{p}_j)\thinspace w_{l+1,s,i}\thinspace \dac{l+1}{s}{p}{j} &\\
&=\phi_{l}'(\csup{y}{l}_i)\sum_{s=1}^{h_{l+1}}\sum_{j=1}^{h_{p}} w_{p+1,t,j}\thinspace \phi'_{p} (\csup{y}{p}_j)\thinspace w_{l+1,s,i}\thinspace \dac{l+1}{s}{p}{j} &\\
&=\phi_{l}'(\csup{y}{l}_i)\sum_{s=1}^{h_{l+1}} w_{l+1,s,i}\sum_{j=1}^{h_{p}} w_{p+1,t,j}\thinspace \phi'_{p} (\csup{y}{p}_j)\thinspace\dac{l+1}{s}{p}{j} &\\
&=\phi_{l}'(\csup{y}{l}_i)\sum_{s=1}^{h_{l+1}} w_{l+1,s,i}\thinspace\dac{l+1}{s}{p+1}{t}. & \mbox{by definition (20)}
\end{align*}

So the result also holds for $r=p+1$, and the inductive step is complete.
\end{proof}

For our purposes here only the derivative amplification coefficients to the final output layer are of interest. Therefore only a special case $r=L$ of the above theorem is actually used:

\begin{cor}
For $\thinspace 1\le l< L$ and $1\le i\le h_{l}$, $1\le o\le p$, \begin{equation}
\dac{l}{i}{L}{o}=\phi_{l}'(\csup{y}{l}_i)\sum_{s=1}^{h_{l+1}} w_{l+1,s,i}\thinspace \dac{l+1}{s}{L}{o}.\end{equation}
\end{cor}

Based on Corollary 3, the derivative amplification coefficients to the nodes in the output layer can be computed incrementally layer by layer in the backward direction, with each layer doing much less work than based on the definition directly. The computational cost can therefore be dramatically reduced.

\section{Error Backpropagation}

As said, to find all the partial derivatives ${\partial f^{[L]}_o\over\partial w_{l,i,j}}(\boldx;\boldw)$ for all $l,i,j,o$, a total of $q\sum\limits_{l=1}^{L-1} h_l$ derivative amplification coefficients need to be maintained and computed. Can we make this number smaller? If $q>1$, then the answer is yes. The key idea is that actually we do not need to compute the partial derivatives of output functions with respect to the weights, what we really need to compute is the partial derivatives of the error with respect to the weights. So in this section we shall develop an error-oriented modification of backpropagation. When $q=1$, this error-oriented modification makes no loss (albeit no improvement either), therefore it is also used in that case for sake of style consistency.

In virtually all applications of neural networks, the error (or loss, cost, penalty, etc.) $E$ that we want to minimize, is a function of $\csup{\boldf}{L}(\boldx;\boldw)$. (Note that a function of $\csup{\boldg}{L}(\boldx;\boldw)$ is also a function of $\csup{\boldf}{L}(\boldx;\boldw)$.) Therefore generally we have \begin{equation}
{\partial E\over \partial w_{l,i,j}}(\boldx;\boldw)=\sum_{o=1}^{q} \varepsilon_o(\boldx;\boldw)\cdot {\partial \csup{\boldf}{L}\over\partial w_{l,i,j}}(\boldx;\boldw).\end{equation} We have seen one particular example of (25), which is (12) in the regression setting. The $\varepsilon_o(\boldx;\boldw)$'s can be computed before backpropagation begins.

Now, let us give a definition of error coefficients for the nodes of the neural network.

For all $1\le l\le L$ and $1\le i\le h_l$, the {\em error coefficient\/} of node $(l,i)$, denoted by $\ec{l,i}$, is defined as \begin{equation}
\ec{l,i}\eqdef\sum_{o=1}^{q} \varepsilon_o(\boldx;\boldw)\cdot\dac{l}{i}{L}{o}.
\end{equation}

Note also that $\ec{l,i}$, like $\dac{l}{i}{r}{t}$, also depend on $(\boldx;\boldw)$.

It is apparent that by definition (26) and definition (19), we have for $1\le o\le q$, \begin{equation}
\ec{L,o}= \varepsilon_o(\boldx;\boldw).
\end{equation}

Further,

\begin{thm}
For $1\le l\le L$, $1\le i\le h_l$, and $1\le j\le h_{l-1}$,\begin{equation}
{\partial E\over \partial w_{l,i,j}}(\boldx;\boldw)=\left\{\begin{array}{cl}
\ec{l,i}, & \mbox{if }j=0, \\
\ec{l,i}\cdot \csup{z}{l-1}_j, & \mbox{otherwise.}\end{array}\right.
\end{equation}
\end{thm}

\begin{proof}
It follows from Theorem 1 that for each $o$ with $1\le o\le q$,
\begin{equation}{\partial f^{[L]}_o\over \partial w_{l,i,j}}(\boldx;\boldw)=\dac{l}{i}{L}{o}\cdot {\partial f^{[l]}_i\over \partial w_{l,i,j}}(\boldx;\boldw).\end{equation}

Thus\begin{align*}
{\partial E\over \partial w_{l,i,j}}(\boldx;\boldw)&=\sum_{o=1}^{q} \varepsilon_o(\boldx;\boldw)\cdot {\partial f^{[L]}_o\over\partial w_{l,i,j}}(\boldx;\boldw) & \mbox{by (25)}\\
&=\sum_{o=1}^{q} \varepsilon_o(\boldx;\boldw) \cdot \dac{l}{i}{L}{o}\cdot {\partial f^{[l]}_i\over \partial w_{l,i,j}}(\boldx;\boldw) & \mbox{by (21)}\\
& = \ec{l,i}\cdot {\partial f^{[l]}_i\over \partial w_{l,i,j}}(\boldx;\boldw)&\mbox{by definition (26)}\\
&=\left\{\begin{array}{cl}
\ec{l,i}, & \mbox{if }j=0, \\
\ec{l,i}\cdot \csup{z}{l-1}_j, & \mbox{otherwise.}\end{array}\right. &\mbox{by (15,18)}
\end{align*}
\end{proof}

\begin{thm}
For $\thinspace 1\le l< L$ and $1\le i\le h_{l}$, \begin{equation}
\ec{l,i}=\phi_{l}'(\csup{y}{l}_i)\sum_{s=1}^{h_{l+1}} w_{l+1,s,i}\thinspace \ec{l+1,s}.\end{equation}
\end{thm}

\begin{proof}
\begin{align*}
\ec{l,i}&=\sum_{o=1}^{q} \varepsilon_o(\boldx;\boldw)\cdot\dac{l}{i}{L}{o}&\mbox{by defintion (26)}\\
&=\sum_{o=1}^{q} \varepsilon_o(\boldx;\boldw)\cdot\phi_{l}'(\csup{y}{l}_i)\sum_{s=1}^{h_{l+1}} w_{l+1,s,i}\thinspace \dac{l+1}{s}{L}{o}&\mbox{by Corollary 3}\\
&=\phi_{l}'(\csup{y}{l}_i)\sum_{o=1}^{q} \varepsilon_o(\boldx;\boldw)\sum_{s=1}^{h_{l+1}} w_{l+1,s,i}\thinspace \dac{l+1}{s}{L}{o}&\\
&=\phi_{l}'(\csup{y}{l}_i)\sum_{o=1}^{q}\sum_{s=1}^{h_{l+1}} \varepsilon_o(\boldx;\boldw)\cdot w_{l+1,s,i}\cdot \dac{l+1}{s}{L}{o}&\\
&=\phi_{l}'(\csup{y}{l}_i)\sum_{s=1}^{h_{l+1}}\sum_{o=1}^{q} \varepsilon_o(\boldx;\boldw)\cdot w_{l+1,s,i}\cdot \dac{l+1}{s}{L}{o}&\\
&=\phi_{l}'(\csup{y}{l}_i)\sum_{s=1}^{h_{l+1}}w_{l+1,s,i}\thinspace\sum_{o=1}^{q} \varepsilon_o(\boldx;\boldw)\cdot \dac{l+1}{s}{L}{o}&\\
& =\phi_{l}'(\csup{y}{l}_i)\sum_{s=1}^{h_{l+1}} w_{l+1,s,i}\thinspace \ec{l+1,s}. &\mbox{by defintion (26)}
\end{align*}
\end{proof}

Based on all the foregoing results, we are now able to provide a full description of the backpropagation algorithm for the regression application.

\begin{algorithm}[h]
\DontPrintSemicolon
\caption{BP\_REG}
\SetKwInOut{Input}{Input}
\SetKwInOut{Output}{Output}
\Input{$\enspace L,n=h_0,h_1,\cdots,h_{L-1},q=h_L,\boldw\in\mathbb{R}^{\sum\limits_{l=1}^L (1+h_{l-1})h_l},\boldx\in\mathbb{R}^n,\boldd\in\mathbb{R}^q$}
\Output{$\enspace {\partial E\over \partial w_{l,i,j}}(\boldx;\boldw)$ for $1\le l\le L$, $1\le i\le h_l$, $0\le j\le h_{l-1}$, where $E$ is defined by (12)}

Compute $\csup{y}{l}_i$ and $\csup{z}{l}_i$ for $1\le l\le L$, $1\le i\le h_l$ based on (14--18)

\For{$o=1,\ldots,q$}{
      $\ec{L,o}=\csup{y}{L}_o-d_o$\;
}

\For{$l=L-1,\ldots,1$}{
    \For{$i=1,\ldots,h_{l}$}{
        $\ec{l,i}=\phi'_l(\csup{y}{l}_i)\sum\limits_{s=1}^{h_{l+1}} w_{l+1,s,i}\thinspace \ec{l+1,s}$\;
    }
}

\For{$l=1,\ldots,L$}{
    \For{$i=1,\ldots,h_l$}{
        ${\partial E\over \partial w_{l,i,0}}(\boldx;\boldw)=\ec{l,i}$\;
        \For{$j=1,\ldots,h_{l-1}$}{
            ${\partial E\over \partial w_{l,i,j}}(\boldx;\boldw)=\ec{l,i}\cdot \csup{z}{l-1}_j$\;
        }
    }
}
\end{algorithm}

We have two remarks for the algorithm: \begin{itemize}
\item This algorithm description is for regression. For other applications of neural networks one can suitably modify line 3. It should not be difficult if one understands the underlying mathematics of the algorithm.

\item Line 1 constitutes the forward propagation part. Lines 2--6 constitute the backpropagation part, which is for computing the $\sum\limits_{l=1}^L h_l$ error coefficients. Lines 7-11 constitute the third finishing part, and we have a much larger freedom in choosing its internal order. The order between the three major parts cannot be altered, except that the third part can be merged into the backpropagation part.
\end{itemize}

\section{Conclusion}
We have presented a new derivation of the classic backpropagation algorithm based on the concept of derivative amplification coefficients first proposed in \cite{cheng2017a}. It is considered that the use of derivative amplification coefficients is essential as a theoretic tool, because without that tool we are unable to use mathematical induction, which is essential in establishing Theorem 2, the center of our derivation. Our new derivation is rigorous, simple, and elegant, and we believe will be helpful for a large portion of practitioners in the neural networks community.

\bibliographystyle{unsrt}

\end{document}